\documentclass[11pt]{article}

\usepackage{amsfonts,amstext,amsmath,amssymb}
\usepackage{color}
\usepackage{graphicx, graphics}
\usepackage{cite}
\usepackage{color}
\usepackage{geometry}
\usepackage{multirow}
\usepackage{stmaryrd}
\usepackage{ulem}
\usepackage{fancyheadings}

\newtheorem{proposition}{Proposition}[section]

\newtheorem{corollary}{Corollary}[section]
\newtheorem{proof}{Proof}[section]

\newcommand{\Mcal}{\mathcal{M}}
\def\argmin{\mathop{\mathrm{argmin}}}

\title{Variational Bayes approach for model aggregation in unsupervised classification with Markovian dependency}

\author{Stevenn Volant$^{1,2}$, Marie-Laure Martin Magniette$^{1,2,3,4,5}$ and
  St\'ephane Robin$^{1,2}$}

\begin{document}

\maketitle
\begin{center}
$^{1}$AgroParisTech, 16 rue Claude Bernard, 75231 Paris Cedex 05, France. \\ 
$^{2}$INRA UMR MIA 518, 16 rue Claude Bernard, 75231 Paris Cedex 05, 
France. \\ 
$^{3}$INRA UMR 1165, URGV, 2 rue Gaston Cr\'emieux, CP5708, 91057, Evry 
Cedex, France. \\ 
$^{4}$UEVE, URGV, 2 rue Gaston Cr\'emieux, CP5708, 91057, Evry Cedex, 
France. \\ 
$^{5}$CNRS ERL 8196, URGV, 2 rue Gaston Cr\'emieux, CP5708, 91057, Evry 
Cedex, France. 
\end{center}

\begin{abstract}

We consider a binary unsupervised classification problem where each observation is associated with an unobserved label that we want to retrieve. More precisely, we assume that  there are two groups of observation: normal and abnormal. The `normal' observations are coming from a known distribution whereas the distribution of the `abnormal' observations is unknown. Several models have been developed to fit this unknown distribution. In this paper, we propose an alternative based on a mixture of Gaussian distributions. The inference is done within a variational Bayesian framework and our aim is to infer the posterior probability of belonging to the class of interest. To this end, it makes no sense to estimate
the mixture component number since each mixture model provides more or less relevant information to the posterior probability estimation. By
computing a weighted average (named aggregated estimator) over the model collection, Bayesian Model Averaging (BMA) is one way of combining models in order to account for information provided by each model. The aim is then the estimation of the weights and the posterior probability for one specific model. In this work, we derive optimal approximations of these quantities from the variational theory and propose other approximations of the weights. To perform our method, we consider that the data are dependent (Markovian dependency) and hence we consider a Hidden Markov Model. A simulation study is carried out to evaluate the accuracy of the estimates in terms of classification. We also present an application to the analysis of public health surveillance systems.\\

\textbf{Keywords:} Model averaging, Variational Bayes inference, Markov Chain, Unsupervised classification.
\end{abstract}

\newpage

\section{Introduction}

\paragraph{Binary unsupervised classification}

We consider an unsupervised classification problem where each
observation is associated with an unobserved label that we want to
retrieve. Such problems occur in a wide variety of domains, such as
climate, epidemiology (see Cai et al.\cite{Cai2009}), or genomics (see McLachlan et al. \cite{McLachlan2002}) where we want to distinguish
`normal' observations from abnormal ones or, equivalently, to
distinguish pure noise from signal. In such situations, some prior
information about the distribution of `normal' observations, or about
the distribution of the noise is often available and we want to take
advantage of it. \\
More precisely, based on observations $X = \{X_t\}$, we want to
retrieve the unknown binary labels $S = \{S_t\}$ associated with each
of them. We assume that `normal' observations (labelled with 0) have
distribution $\phi$, whereas `abnormal' observations (labelled with 1)
have distribution $f$. We further assume that the null distribution
$\phi$ is known, whereas the alternative distribution $f$ is not. In a
classification perspective, we want to compute
\begin{equation} \label{Tt}
T_t = \Pr\{S_t = 0 | X\}.
\end{equation}

\paragraph{Bayesian model averaging (BMA)}
The probability $T_t$
depends on the unknown distribution $f$.
Many models can be considered to fit this distribution and we denote
$\Mcal = \{f_m; m=1, \ldots, M\}$ a finite collection of such models. As none of
these models is likely to be the true one, it seems more natural to
gather information provided by each of them, rather than to try to
select the `best' one. The Bayesian framework is natural
for this purpose, as we have to deal with model uncertainty. \\
Bayesian model averaging (BMA) has been mainly developed by Hoeting
et al. \cite{hoeting_bayesian_1999} and provides the general framework of
our work.  It has been demonstrated that BMA can improve predictive
performances and parameter estimation in Madigan and Raftery
\cite{madigan_model_1993}, Madigan et
al.\cite{madigan_enhancingpredictive_1995}, Raftery et
al.\cite{raftery_bayesian_1997,volinsky_bayesian_1997} or Raftery and
Zheng \cite{Raftery03long-runperformance}. Jaakkola and Jordan
\cite{jaakkola_improvingmean_1998} also demonstrated that model
averaging provides a gain in terms of classification and fitting.  The
determination of the weight $\alpha_m$ associated with each model $m$
when averaging is a key ingredient of all these approaches.

\paragraph{Weight determination}
 
As shown in Hoeting et al. \cite{hoeting_bayesian_1999} the standard
Bayesian reasoning leads to $\alpha_m = \Pr\{M = m | X\}$, where $M$
stands for the model. In a classical context, the calculation of
$\alpha_m$ requires one to integrate the joint conditional distribution $P(M, \Theta
|X)$, where $\Theta$ is the vector of model parameters, and several approaches can be used. The BIC criterion (Schwarz
\cite{schwarz_estimatingdimension_1978}) is based on a Laplace
approximation of this integral, which is questionable for small
sample sizes. One other classical method is the MCMC (Monte Carlo Markov Chain) \cite{Andrieu03} which samples the distribution and can provide 
an accurate estimation of the joint conditional, but at the cost of huge (sometimes prohibitive) computational time. \\
In the unsupervised classification context, the problem is even more
difficult as we need to integrate the conditional $P(M, \Theta, S |X)$
since the labels are unobserved.  This distribution is generally not
tractable but, for a given model, Beal and
Ghahramani\cite{beal_variational_2003} developed a variational Bayes
strategy to approximate $P(\Theta, S |X)$.  Variational techniques aim
at minimising the Kullback-Leibler (KL) divergence between $P(\Theta,
S|X)$ and an approximated distribution $Q_{\Theta, S}$ (Wainwright and
Jordan\cite{Wainwright2008}, Corduneanu and Bishop\cite{Corduneanu2001}). Jaakkola and Jordan \cite{jaakkola_improvingmean_1998} proved that the variational approximation can be improved by using a mixture of distributions rather than factorised distribution as the approximating distribution. A mixture distribution $Q_{mix}$ is chosen to minimise the KL-divergence with respect to $P(\Theta,S|X)$. Unfortunately, they need to average the log of $Q_{mix}$ over all the configurations which leads to untractable computation and a costly algorithm involving a smoothing distribution must be implemented.

\paragraph{Our contribution}

In this article, we propose variational-based weights for model
averaging, in presence of a Markov dependency between the unobserved
labels.  We prove that these weights are optimal in terms of
KL-divergence from the true conditional distribution $P(M | X)$. To this end, we optimise the KL-divergence between $P(\Theta,
S,M|X)$ and an approximated distribution $Q_{\Theta, S,M}$ (Section \ref{VarWei}). This optimisation problem differs from that of Jaakkola and Jordan (see equation 14 in  \cite{jaakkola_improvingmean_1998}). Based on the approximated distribution of $P(\theta, S| M, X)$, we derive other estimations of the weights. \\
We then go back to the specific case of unsupervised classification and consider a collection $\mathcal{M}$ of mixtures of  parametric exponential family distributions (Section \ref{MixMod}). We propose a
complete inference procedure that does not require any specific
development in terms of inference algorithm.  In order to assess our
approach, we propose a simulation study which highlights the gain of
model averaging in terms of binary classification (Section \ref{Simul}). We also present  an application to the analysis of public health surveillance systems (Section \ref{Appli}).



\section{Variational weights}
\label{VarWei}
\subsection{A two-step optimisation problem}

In a Bayesian Model Averaging context, we focus
on averaged estimator to account for model uncertainty It implies evaluating the conditional distribution:
\begin{eqnarray}
P(M|X) = \int P(H,M|X) dH,
\label{TrueW}
\end{eqnarray}
where $H$ stands for all hidden variables, that is $H=(S,\Theta)$, and $M$ denotes the model.\\

In order to calculate this distribution, we need to compute the joint posterior distribution of $H$ and $M$. Due to the latent structure of the problem this is not feasible but the mean field/variational theory allows one to derive an approximation of this distribution. It has mainly been developed by Parisi \cite{Parisi88} and provides an alternative approach to MCMC for inference problem within a Bayesian framework.
The variational approach is based on the minimisation of the KL-divergence between $P(H,M|X)$ and an approximated distribution $Q_{H,M}$. The optimisation problem can be decomposed as follows: 
\begin{eqnarray}
\min_{Q_{H,M}} KL(Q_{H,M}||P(H,M|X))
&& =  \min_{Q_{M}} \left[KL(Q_{M}||P(M|X))\right. \nonumber \\
&&\left.+ \sum_m Q_M(m) \min_{Q_{H|m}} KL(Q_{H|m}||P(H|X,m)) \right].
\label{ourKL}
\end{eqnarray}
This decomposition separates $Q_{M}$ and $Q_{H|M}$, and so these optimisations can be realised independently. We are mostly interested in $Q_{M}$ which provides an approximation of $P(M|X)$ given in Equation \ref{TrueW}. Furthermore, since the collection $\mathcal{M}$ is finite, we do not need to put any restriction on the form of $Q_{M}$ and may deal with the weights $\alpha_m=Q_{M}(m)$ for each $m \in \mathcal{M}$.
In the following, we will first minimise the KL-divergence with regard to $Q_{M}$ leading to weights that depend on $Q_{H|m}$. In a second step, we will consider the approximation of $P(H|X,m)$. 

\subsection{Weight function of any approximation of $P(H|X,m)$}
\label{Optimal variational weights}
We now consider the optimisation of $Q_{M}$. Proposition \ref{prop1} provides the optimal weights.

\begin{proposition}
\label{prop1}
The weights that minimise $KL(Q_{H,M}||P(H,M|X))$ with respect to $Q_{M}$, for given distributions $\{Q_{H|m}, m \in \mathcal{M}\}$, are
$$
\overline{\alpha}_m(Q_{H|m}) \propto P(m) \exp[-KL(Q_{H|m}||P(H|X,m)) + \log P(X|m)],
$$
with $\sum_m \overline{\alpha}_m(Q_{H|m}) = 1$.
\end{proposition}

\begin{proof}
$KL(Q_{H,M}||P(H,M|X))$ can be rewritten as: 
\begin{eqnarray*}
&&\sum_m \int  Q_{H|m}(h) Q_{M}(m) \log \left[ \frac{Q_{H|m}(h) Q_{M}(m)}{P(h,m,X)/P(X)}\right] dh\\ 
&=&\sum_m \int  Q_{H|m}(h) Q_{M}(m) \left[  \log Q_{H|m}(h)+ \log Q_{M}(m)+ \log P(X) - \log P(h,m,X)  \right] dh\\ 
& = & \sum_m \left(\int  Q_{H|m}(h) Q_{M}(m) \left[  \log \frac{Q_{H|m}(h)}{P(h,X|m)} +  \log Q_{M}(m)- \log P(m) \right]dh\right) + \log P(X)\\   
& = &\sum_m \left(Q_{M}(m) \left[  KL ( Q_{H|m}||P(H,X|m)) + \log Q_{M}(m)- \log P(m)  \right] \right)+\log P(X)
\end{eqnarray*}
The miminisation with respect to $Q_{M}$ subject to $\sum_m Q_{M}(m) = 1$ gives the result.
\end{proof} 

Note that if $Q_{H|m}=P(H|X,m)$ then KL-divergence in the exponential is 0, so $\overline{\alpha}_m$ resumes to $P(m|X)$.

\subsection{Weights based on the optimal approximation of $P(H|X,m)$}

We now derive three different weights from the variational Bayes approximation.

\paragraph{Full variational approximation}
To solve the optimisation problem \ref{ourKL} we still need to minimise the divergence \\ $KL(Q_{H|m}||P(H|X,m))$ for each model $m$, where $H=(S,\Theta)$.

Due to the latent structure, the optimisation cannot be done directly. When $P(X,S|\Theta,M)$ belongs to the exponential family and if $P(\Theta|M)$ is the conjugate prior, the Variational Bayes EM (VBEM: Beal and Ghahramani\cite{beal_variational_2003}) algorithm allows us to minimise this KL-divergence within the class of factorised distributions: $\mathcal{Q}_m = \{Q_{H|m}: Q_{H|m} = Q_{S|m}Q_{\Theta|m}\}$. Due to the restriction, the optimal distribution 
$$
Q_{H|m}^{VB} = \arg\min_{Q \in \mathcal{Q}_m} KL(Q_{H|m}||P(H|X,m))
$$
is only an approximation of $P(H|X,m)$. This allows us to define the optimal variational weights.

\begin{corollary}
\label{OW}
The weights $\widehat{\alpha}^{VB}_m$ achieving the optimisation problem \ref{ourKL} for factorised conditional distribution $Q_{H|m}$ are:
$$
\widehat{\alpha}^{VB}_m  \propto  P(m) \exp[-\min_{Q_{H|m} \in \mathcal{Q}_m } KL(Q_{H|m}||P(H|X,m)) + \log P(X|m)].
$$
\end{corollary}

\paragraph{Plug-in weights}
\label{otherW}
The weights $\alpha_m = \Pr\{M = m | X\}$ can be estimated by using a plug-in estimation based on a direct application of Bayes' theorem. The conditional probability $P(m|X)$ is proportional to $P(X|m)$ that equals to  ${P(X|m,\Theta) P(\Theta|m)}/{P(\Theta|X,m)}$ for any value of $\Theta$, which avoids integrating over $S$. The distribution $Q_{\Theta|m}^{VB}$ resulting from the VBEM algorithm is an approximation of $P(\Theta|X,m)$. Setting $\Theta$ at its (approximate) posterior mean $\theta^* = \mathbb{E}_{Q^{VB}_{\Theta}}(\Theta)$, we define the following plug-in estimate
\begin{eqnarray}
\widehat{\alpha}^{PE}_m \propto P(m) \frac{P(X|m,\theta^*) P(\theta^*|m)}{Q_{\Theta|m}^{VB}(\theta^*)}.
\end{eqnarray}

\paragraph{Importance sampling}

The weights given in Corollary \ref{OW} are based on an approximation of the conditional distribution $P(H|X)$. But, the weights defined in \ref{TrueW} can be estimated via importance sampling (Marin and Robert \cite{marin_importance_2009}). For any distribution $R$, we have
\begin{eqnarray*}
P(m|X) \propto \int P(m) \frac{P(X|h,m)P(h|m)}{R(h)} R(h) dh.
\label{ImpSamp2}
\end{eqnarray*}
Importance sampling provides an unbiased estimator of $P(m|X)$. The importance function $R$ can be chosen to minimise the variance of the estimator. The minimal variance is reached when $R(H)$ equals $P(H|X)$ \cite{marin_importance_2009}. Thus, in the variational framework, the approximated posterior distribution $Q_{H|m}^{VB}$ is a natural choice for the importance function $R$, leading to the following weights:
\begin{eqnarray*}
 \widehat{\alpha}^{IS}_m \propto P(m) \frac{1}{B} \sum_{b=1}^B \frac{P(X|H^{(b)},m)P(H^{(b)})}{Q_{H|m}^{VB}(H^{(b)})},
\qquad
\{H^{(b)}\}_{b = 1, \ldots, B} \mbox{ i.i.d. } \sim Q_{H|m}^{VB}.
\end{eqnarray*}
Although this estimate is unbiased, when the number of observations is large, it may require a long computational time to get a reasonably small variance.\\


\section{Unsupervised classification}
\label{MixMod}

\subsection{Binary hidden Markov model}
We now come back to the original binary classification problem with Markov dependence between the labels. To this aim we consider a classical hidden Markov model (HMM). We assume that $\{S_t\}_{1 \leq t \leq n}$ is a first order Markov chain with transition matrix $\Pi = \{\pi_{ij}; i,j=0,1\}$. The observed data $\{X_t\}_{1 \leq t \leq n}$ are independent conditionally to the labels. We denote $\phi$ the emission distribution in state 0 ('normal') and $f$ the emission distribution in state 1 (`abnormal'). We recall that the function $\phi$ is known whereas $f$ is unknown and we consider the collection $\Mcal = \{f_m; m=1, \ldots, M\}$ where $f_m$ is a mixture of $m$ components:
$$
f_m(x)  = \sum_{k=1}^m p_k \phi_k(x), \qquad \mbox{with } \sum_{k=1}^m p_k = 1.
$$ 
This collection is large as it allows us to fit the data from a two-component mixture (see McLachlan et al. \cite{McLachlan2002})  to a semi-parametric kernel-based density (see Robin et al. \cite{robin_semi-parametric_2007}). When $f$ is approximated by a mixture of $m$ components, the initial binary HMM with latent variable $S$ can be rephrased as an $(m+1)$-state HMM with hidden Markov chain $\{Z_t\}$ taking its values in $\{0, \dots, m\}$ with transition matrix
$$
\Omega = 
\left(
\begin{array}{cccc}
\pi_{00} & \pi_{01}p_1 & \ldots &  \pi_{01}p_m \\
\pi_{10} & \pi_{11}p_1 & \ldots &  \pi_{11}p_m \\
\vdots & \vdots  & \vdots  & \vdots \\
\pi_{10} & \pi_{11}p_1 & \ldots &  \pi_{11}p_m 
\end{array}
\right).
$$
The observed data $\{X_t\}_{1 \leq t \leq n}$ are independent conditionally to the $\{Z_t\}$ with distribution 
$$
X_t|Z_t \sim \phi_{Z_t},
$$ 
where $\phi_0 = \phi$.
Hence, we have two latent variables $Z$ and $S$ which correspond to the group within the whole mixture and to the binary classification, respectively. 

\subsection{Variational Bayes inference}

The VBEM (Beal and Ghahramani\cite{beal_variational_2003}) aims at minimising the KL-divergence in exponential family/conjugate prior context. The quality of the VBEM estimators has been studied in Wang and Titterington (\cite{wang_convergence_2004},\cite{wang_inadequacy_2004},\cite{Wang03}) for mixture models. Wang and Titterington \cite{and_lack_2003} have also studied the quality of variational approximation for state space models. The VBEM algorithm has been studied by McGrory and Titterington\cite{mcgrory_variational_2006} for the HMM with emission distributions belonging to the exponential family.
In these articles, the authors have demonstrated the convergence of the variational Bayes estimator to the maximum likelihood estimator, at rate $\mathcal{O}(1/n)$. They also show that the covariance matrix of the variational Bayes estimators is underestimated compared to the one obtained for the maximum likelihood estimators. 

In our case, $P(X, S | \Theta, M)$ does not belong to the exponential family whereas  $P(X, Z |\Theta, M)$ does. We will therefore make the inference on the $(m+1)$-state hidden Markov model involving $Z$ rather than the binary hidden Markov model involving $S$. Despite the specific form of the transition matrix $\Omega$, it does not modify the framework of the exponential family/conjugate prior. To be specific, $\log P(X, Z | \Theta, M)$ can be decomposed as $\log P(Z | \Theta, M)  + \log P(X | Z, \Theta, M)$ and only the first term involves $\Omega$:

\begin{eqnarray}
\label{PZ}
	\log P(Z|\Theta,M) & = & \sum_{k=1}^m \sum_{j=1}^m N_{kj} \log \pi_{11} + N_{00} \log \pi_{00} + \sum_{k=1}^m N_{k0} \log \pi_{10} \nonumber \\
&& + \sum_{j=1}^m N_{0j} \log \pi_{01} + \sum_{k=1}^m Z_{1k} \log q_1 + Z_{10} \log q_0  \nonumber \\
&&  +\sum_{k=0}^m \sum_{j=1}^m N_{kj} \log p_j +\sum_{k=1}^m Z_{1k} \log p_k,
\end{eqnarray}
with $N_{kj} = \sum_{t \geq 2} Z_{t-1,k} Z_{tj}$ and $q$ is the stationary distribution of $\Pi$.
Since $\log P(Z | \Theta, M)$ can be written as a scalar product $\Phi.u(Z)$ with $\Phi$ the vector of parameters and $u(Z)$ the vector containing the  $\{N_{kj}\}_{1 \leq k,j \leq m}$ and the sums over $Z$, it shows that $Z | \Theta, M$ belongs to the exponential family and that this specific form of $\Omega$ only affects the updating step of hyper-parameters.

\subsection{Model averaging}

For each model $m$ from the collection $\mathcal{M}$, the VBEM algorithm provides the optimal distributions $Q^{VB}_{H|m}$, from which we can derive the three weights defined in Section \ref{VarWei}: $\widehat{\alpha}_m^{VB}$, $\widehat{\alpha}_m^{PE}$ and $\widehat{\alpha}_m^{IS}$. Based on these weights, we can get an averaged estimate of the distribution $f$:
$$
\widetilde{f}^{\mathcal{A}} = \sum \widehat{\alpha}^{\mathcal{A}}_m \widehat{f}_m,
$$
where $\mathcal{A}$ corresponds to one of the proposed approaches (VB, PE or IS). Although the largest model only involves $M$ components, the averaged distribution is a mixture with ${M(M+1)}/{2}$ components. As we are mostly interested in the estimation of the posterior probability $T_t$ defined in \ref{Tt}, we  similarly define its averaged estimate:
$$
\widetilde{T}_t^{\mathcal{A}} = 1 - \sum_m \widehat{\alpha}^{\mathcal{A}}_m \mathbb{E}_{Q^{VB}_{Z|m}}(S_{t}),
$$

where $\mathbb{E}_{Q^{VB}_{Z|m}}(S_{t})$ corresponds to the expected  value of $S$ calculated with the optimal variational posterior distribution of $Z$. This expectation does not depend on $\mathcal{A}$.


\section{Simulation study}
\label{Simul}

In this section, we study the efficiency of the estimators defined in the previous sections. First, we study the accuracy of $\alpha^{VB}$ and $\alpha^{PE}$ in terms of weight estimation. Then, we focus on the accuracy from a classification point of view. We therefore liken the averaged estimator of the posterior probability $T_t$ to the theoretical one. We also compare the averaging approach with a classical two-state HMM and with the HMM which has the highest weight calculated with the importance sampling approach, called throughout the paper ``selected HMM''.

 \subsection{Simulation design}
\label{Simulation design}
We simulate a binary HMM as described in Section \ref{MixMod}, where $f$ is non Gaussian and define as the probit transformation of  a uniform-distribution on $[0,\frac{1}{c}]$, with c $\in [5,7,10,15]$. The difficulty of the problem decreases with the parameter $c$. We also consider four different transition matrices which have the same form given by:

\begin{eqnarray}
\Pi_u = 
\left(
\begin{array}{cc}
1-lu & lu \\
l(1-u) & 1-l(1-u) \\
\end{array}
\right)
\end{eqnarray}
where $l$ is the shifting rate which varies from $0$ to $1$ and $u$ corresponds to the proportion of the group of interest and is chosen within $\{0.05,0.1,0.2,0.3\}$. For each of the 16 configurations we generate $P=100$ samples of size $n=100$.  The inference is done in a semi-homogeneous case: for each simulation condition, we fit a 7-component Gaussian mixture with common variance $\sigma^2$ and mean $\mu_k$ for the alternative. In a Bayesian context, the parameters are random variables with prior distributions. These distributions are chosen to be consistent with the exponential conjugate family. We denote by $\lambda$ the precision parameter, $\lambda = \frac{1}{\sigma^2}$, we have:

 \begin{itemize}
	\item Transition matrix: For $j=1,2$, $\pi_{j.} \sim \mathcal{D}(1,1)$.
        \item Mixture proportions: $p \sim \mathcal{D}(1, \ldots, 1)$.
	\item Precision: $\lambda \sim \Gamma(0.01,0.01)$.
	\item Means: $\mu_k|\lambda \sim \mathcal{N}\left(0,\frac{1}{0.01 \times \lambda}\right)$
\end{itemize}

\subsection{Results}

We present the results for $l=0.6$. We considered other values for this parameter but the performances are almost similar.  

\subsubsection{Accuracy of the weight}

We consider the importance sampling as a reference for weight estimation as it provides an unbiased estimate of the true weights whatever the approximation. We compared it to VB and PE weights by calculating the total variation distance, which  quantifies the dissimilarity between two distributions  $\alpha^1$ and $\alpha^2$: 

\begin{eqnarray}
\delta(\alpha^1,\alpha^2) = \frac{1}{2} \sum_x |\alpha^1(x) - \alpha^2(x)|.
\label{totalvaria}
\end{eqnarray}
The closer to $0$ this distance is, the better the estimation of the weights.

\begin{table}[h!]
\scriptsize
\begin{center}
$$
\begin{array}{|c|c|c||c|c||c|c||c|c|}
\hline
& \multicolumn{8}{c|}{u}\\
\hline
 u & \multicolumn{2}{c||}{0.05} &\multicolumn{2}{c||}{0.1}&\multicolumn{2}{c||}{0.2}&\multicolumn{2}{c|}{0.3}\\
\hline
c & PE & VB & PE & VB & PE & VB & PE & VB \\
\hline
5      &0.419  &0.069 	&0.370  &0.101	&0.453  &0.120  &0.456  &0.069\\
7      &0.438  &0.096 	&0.403  &0.101  &0.287  &0.101  &0.257  &0.072\\
10     &0.386  &0.092 	&0.271  &0.180  &0.232  &0.115  &0.107  &0.092\\
15     &0.372  &0.093	&0.303  &0.158	&0.258  &0.129  &0.102  &0.101\\

\hline
\end{array}
$$
\caption{Total variation distance between the estimated weights with respect to importance sampling for each value of $u$ and $c$.}
\label{DistWeight}
\end{center}
\end{table}

Table \ref{DistWeight} shows that VB weights are  the closest to IS weights. The total variation distance $\delta(\alpha^{VB},\alpha^{IS})$ is close to 0 whatever the simulation study. In contrast, the PE weights seem not to be correct for approximating the true weights except when the two populations are well separated.
These trends are also brought out when we focus on the weights calculated for the $P$ samples given a simulation condition. On average, compared to the PE approach, the VB method tends to provide weight estimations close to those of the IS approach. For instance, for $c=7$ and $u=0.2$, they mix three models with a huge weight ($\approx 0.70$) for $f_1$  and weights around $0.15$ for $f_2$ and $f_3$. However, the VB method has more stable estimated weights than IS. 
PE is the more stable approach among the three but it tends to only select the two-component model with an average weight around $0.95$.

\paragraph{Conclusion on the weight estimation}

By directly analysing the weight estimation, the similarities between the IS and the VB methods have clearly appeared. The VB method provides a good estimation of the true weights which is not the case for PE. Hence, when the computational time of the IS method is becoming very high, we get a real advantage by using the VB method in terms of weight estimation.

\subsubsection{Accuracy of the posterior probabilities}

Once the weights have been estimated, the averaged estimates of the posterior probabilities $T_t$ are computed for each approach. The aim of the VB method is to cluster the data into two populations. In many cases, these populations are difficult to distinguish but some observations are easily classifiable without any statistical approach. Hence, we put aside observations  with a theoretical probability of belonging to the cluster of interest smaller than $0.2$ or higher than $0.8$.
A classical indicator to measure the quality of a given classification is the MSE (Mean Square Error) which evaluates the difference between the averaged estimate $\widetilde{T}^{\mathcal{A}}$ of one method of  $\mathcal{A}$ and the theoretical values $T^{(th)}$. 
\begin{eqnarray}
 MSE^{\mathcal{A}} = \frac{1}{P}\sum_{p=1}^{P} \frac{1}{n}\sum_{t=1}^n (\widetilde{T}_{t,p}^{\mathcal{A}}-T_{t,p}^{(th)})^2,
\label{RMSE2}
\end{eqnarray}
The $MSE^{\mathcal{A}}$ estimation allows us to evaluate the quality of the estimates provided by Model $m$ over all datasets $p=\{1,...,P\}$  and one approach of $\mathcal{A}$. The smaller the MSE, the better the performances are.\\

Since we deal with synthetic data, we can look at the best achievable MSE. This aims at minimising the MSE within the averaged estimator family to obtain an oracle weight We denote this oracle by $\alpha^*$ and we have:
\begin{eqnarray}
\alpha^* = \argmin_{\alpha} || T^{(th)} - \sum_{m=1}^M \alpha_m \widehat{T}^{(m)}||_2,
\end{eqnarray}
with $\sum_{m=1}^M \alpha_m = 1$ and $\forall m \in \{1,...,M\}$,  $0\leq \alpha_m \leq 1$. The variable $\widehat{T}^{(m)}$ is the estimation of $T$ supplied by model $m$. This oracle can be viewed as the weights we would choose if the theoretical posterior probability of belonging to the group of interest were known. This oracle estimator is obtained by a functional regression under non-negativity constraint and it can be written as:
\begin{eqnarray}
\alpha^* = (\widehat{T}'\widehat{T})^{-1} \widehat{T}'T^{(th)} \times \gamma
\end{eqnarray}
where $\gamma$ is a normalising constant and $\widehat{T}$ is the matrix containing the estimates $\widehat{T}^{(m)}$ for all model $m$. Several algorithms allow one to calculate this estimator numerically by taking constraints into account. In this article, the optimisation has been achieved by the Newton-Raphson algorithm.

\begin{figure}
\begin{center}
\includegraphics[height=12cm,width=15cm]{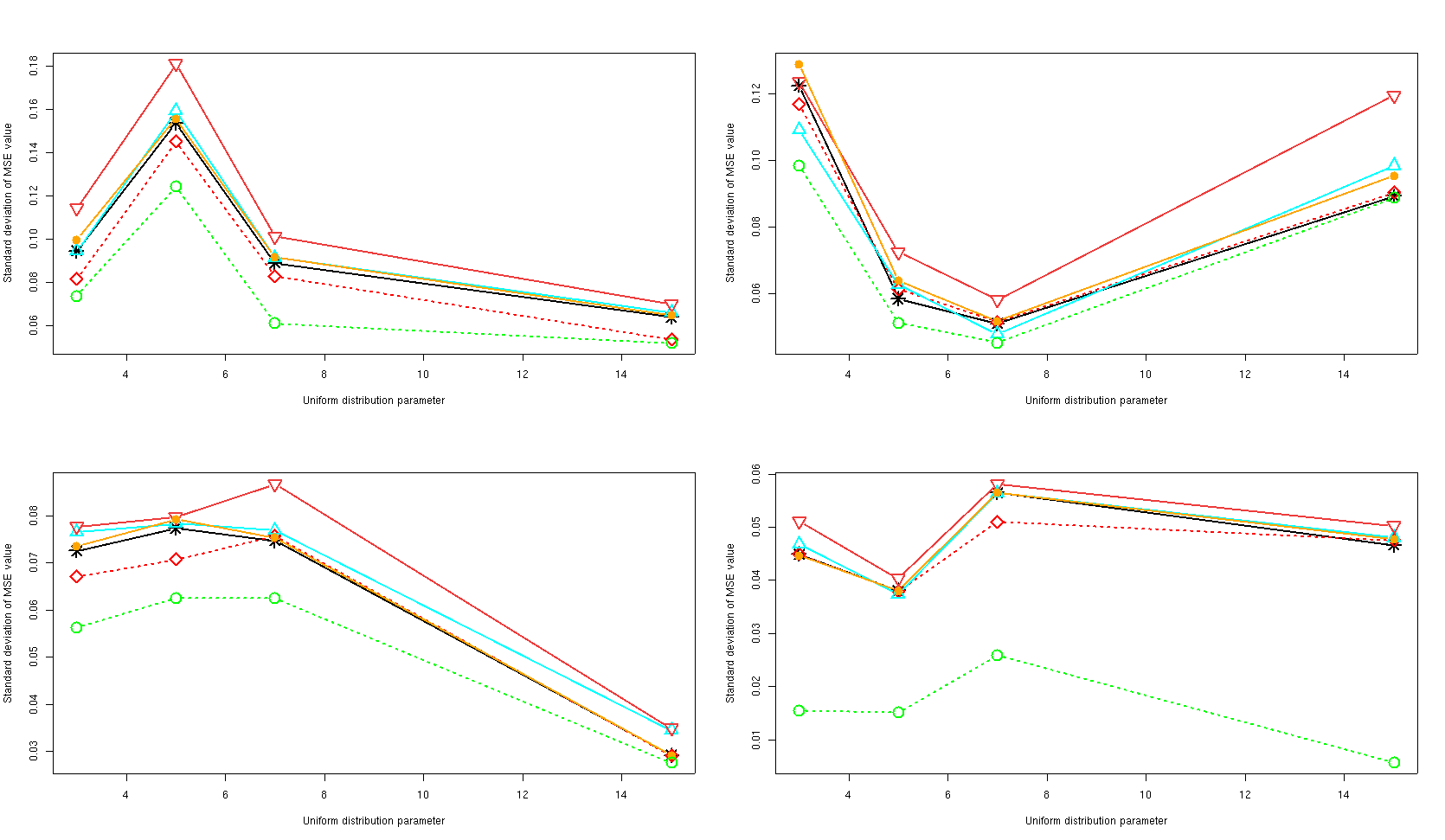}
\caption{Mean square error (MSE) between the true posterior probabilities and the estimates as a function of the uniform parameters. Methods: "$\Delta$": PE, "$\nabla$": two-state-HMM, "*": IS, "$\bullet$":  Selected HMM, "$\Diamond$": VB, "O" : Oracle. Top left: $\Pi_{0.05}$, Top right: $\Pi_{0.1}$, Bottom left: $\Pi_{0.2}$, Bottom right: $\Pi_{0.3}$. VB and Oracle are in dotted lines.}
\label{MSE}
\end{center}  
\end{figure}

Figure \ref{MSE} displays the MSE calculated for the different methods under the various simulation conditions. First, we notice that the VB method based on the optimal variational weights provides good results in most of the cases. Moreover, we observe that an averaging approach with either the IS or VB method provides better results than the selected HMM. We observe that the PE method and the two-state-HMM  provide the worse estimates for many simulation conditions than do the VB and IS methods. 
Another comment is that there is no method which is the best whatever the simulation condition. Moreover, the estimations get closer to the oracle estimator when the problem is becoming easier.\\
Figure \ref{SdMSE} shows the standard deviation of the MSE over all the simulation conditions. We notice that the VB method has one of the lowest variabilities. Once more, the two-state HMM has the worst performances.\\

\begin{figure}
\begin{center}
\includegraphics[height=12cm,width=15cm]{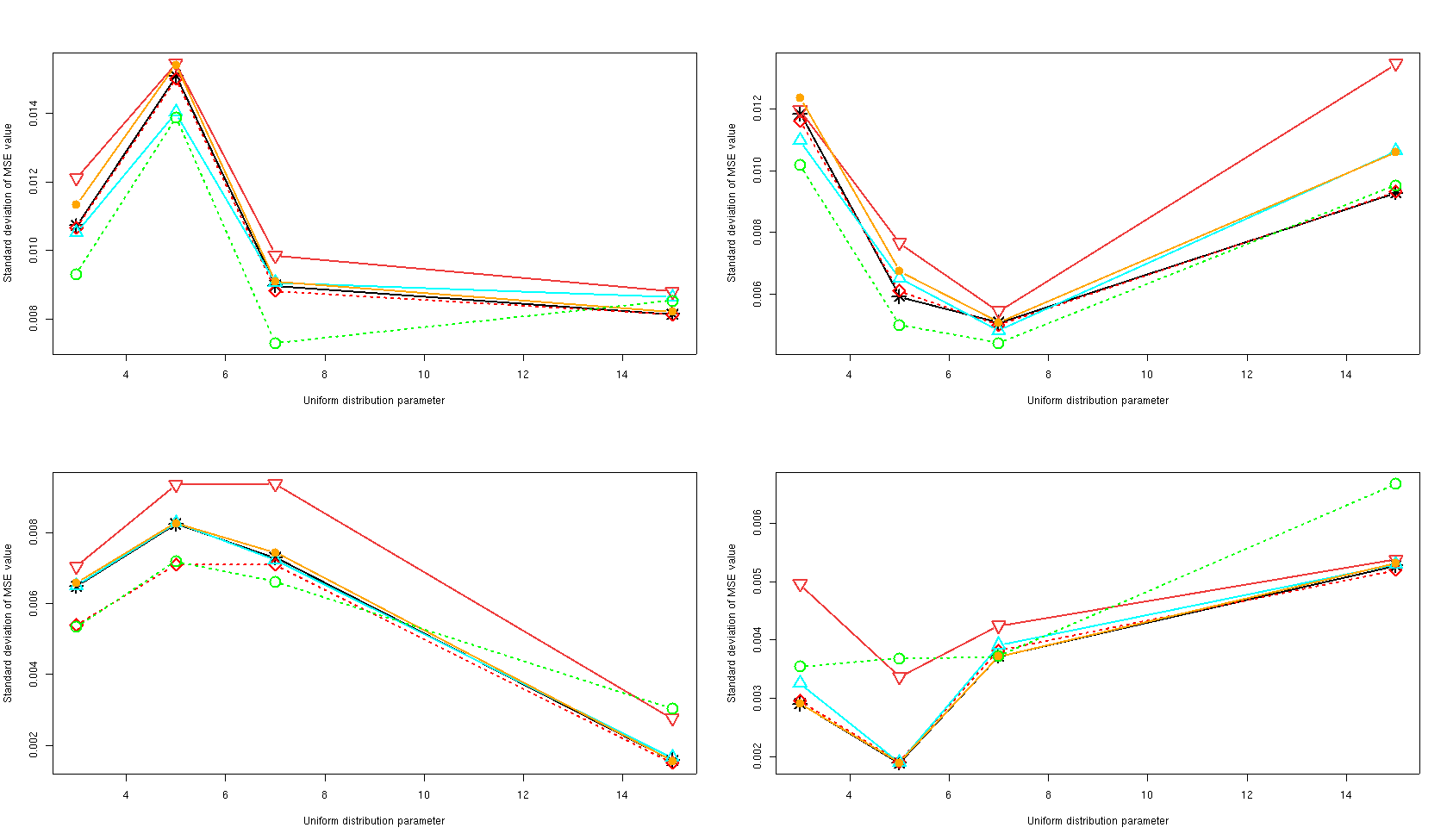}
\caption{Standard deviation of the MSE. Methods: "$\Delta$": PE, "$\nabla$": two-state-HMM, "*": IS, "$\bullet$":  Selected HMM, "$\Diamond$": VB, "O" : Oracle. Top left: $\Pi_{0.05}$, Top right: $\Pi_{0.1}$, Bottom left: $\Pi_{0.2}$, Bottom right: $\Pi_{0.3}$. VB and Oracle are in dotted lines.}
\label{SdMSE}
\end{center}  
\end{figure}

\begin{table}
\scriptsize
\begin{center}
$$
\begin{array}{|c|c|c|c|c|c|}
\hline
& \multicolumn{5}{c|}{u=0.05} \\
\hline
c & PE & VB & IS &Selected HMM & Oracle\\
\hline
5	&0.44 (0.04) &\textbf{0.36 (0.03)} &0.38 (0.04)	& 0.42 (0.03)&	0.31 (0.02)\\
7	&0.54 (0.04) &\textbf{0.42 (0.04)} &0.43 (0.04)	& 0.47 (0.03)	&0.34 (0.02)\\
10	&0.35 (0.04) &\textbf{0.30 (0.04)} &\textbf{0.30 (0.04)} & 0.34 (0.04)&	 0.21 (0.03)\\
15	&0.38 (0.04) &0.34 (0.04) &\textbf{0.33 (0.04)}	& 0.36 (0.03)&	0.23 (0.03)\\
\hline
\hline
 &\multicolumn{5}{c|}{u=0.1}\\
\hline
c & PE & VB & IS &Selected HMM & Oracle\\
\hline
5	&0.40 (0.04) &\textbf{0.37 (0.03)} &0.39 (0.03)& 0.39(0.03)&	 0.29 (0.03)\\
7	&0.29 (0.03) &\textbf{0.23 (0.03)} &\textbf{0.23 (0.03)}& 0.25 (0.03)&	 0.17 (0.02)\\
10	&0.28 (0.03) &0.28 (0.03) &\textbf{0.23 (0.03)}& 0.28 (0.03)&	 0.16 (0.02) \\
15	&0.25 (0.04) &0.22 (0.03) &\textbf{0.20 (0.03)}& 0.22 (0.03)&	 0.17 (0.02) \\
\hline
\hline
&\multicolumn{5}{c|}{u=0.2}\\
\hline
c & PE & VB & IS &Selected HMM & Oracle  \\
\hline
5	&0.33 (0.03) &\textbf{0.29 (0.03)} &0.30 (0.03) & 0.31 (0.03)&	 0.19 (0.02) \\
7	&0.26 (0.03) &\textbf{0.23 (0.02)}  &0.24 (0.02)& 0.25 (0.02)&	0.18 (0.02)\\
10	&0.23 (0.03) &0.20 (0.02)  &\textbf{0.19 (0.02)} & 0.23 (0.01)&	 0.17 (0.02)\\
15	&0.08 (0.01) &0.09 (0.01)  &\textbf{0.07 (0.01)} & 0.09 (0.01)	&0.06 (0.02)\\
\hline	
\hline	
&\multicolumn{5}{c|}{u=0.3}\\
\hline
c & PE & VB & IS &Selected HMM & Oracle  \\
\hline
5	&0.23 (0.02) &\textbf{0.19 (0.01)} &0.20 (0.01)& 0.22 (0.01)&	0.16 (0.01) \\
7	&0.13 (0.01) &\textbf{0.11 (0.01)} &0.12 (0.01) & 0.13 (0.01)&	0.09 (0.01)\\
10	&0.17 (0.02) &0.12 (0.01) &\textbf{0.11 (0.01)} & 0.18 (0.01)	&0.03 (0.01)\\
15	&0.12 (0.01) &0.10 (0.01) &\textbf{0.09 (0.01)} & 0.12 (0.01)&	0.06 (0.01)\\
\hline

\end{array}
$$
\caption{Mean(sd) of the misclassification rate for the three averaging approaches.}
\label{Missclass}
\end{center}
 \normalsize
\end{table}

Table \ref{Missclass} includes information on the misclassification for the three averaging approaches. The misclassification rate is calculated on the $P$ samples whatever the simulation condition. The values in bold correspond to the smallest misclassification rate among the PE, VB and IS approaches. First, we note that the VB and the IS methods have very similar misclassification rates whatever the simulation condition. Moreover, this rate corresponds to the best rate of the three averaging methods.  The averaged estimator supplied by the plug-in weights estimation seems to misclassify more data than the other approaches.  Once again, Table \ref{Missclass} shows us that the VB approach provides good results when the simulation condition is complicated. In fact, when $c$ equals either 5 or 7, the averaging method based on optimal variational weights provides the lowest misclassified rate among the three averaging approaches. Since the misclassification rate of the oracle is close to the rates obtained by VB and IS estimation, the two approaches provide good results for each value of $c$ and $u$. An other comment is that the selected HMM approach always prodives worse results than the IS and VB ones. This means that the averaging approach brings a gain to the posterior probability estimation.

Figure \ref{FigEntropy} shows the entropy of the weights. We note that the optimal variational weights have one of the largest entropies among all the proposed weights. This means that the VB method tends to mix several  models. Contrary to the other  three weights, PE has a low entropy. This method seems to select only one model to infer posterior probability and does not take others into account.\\

\begin{figure}
\begin{center}
\includegraphics[height=12cm,width=15cm]{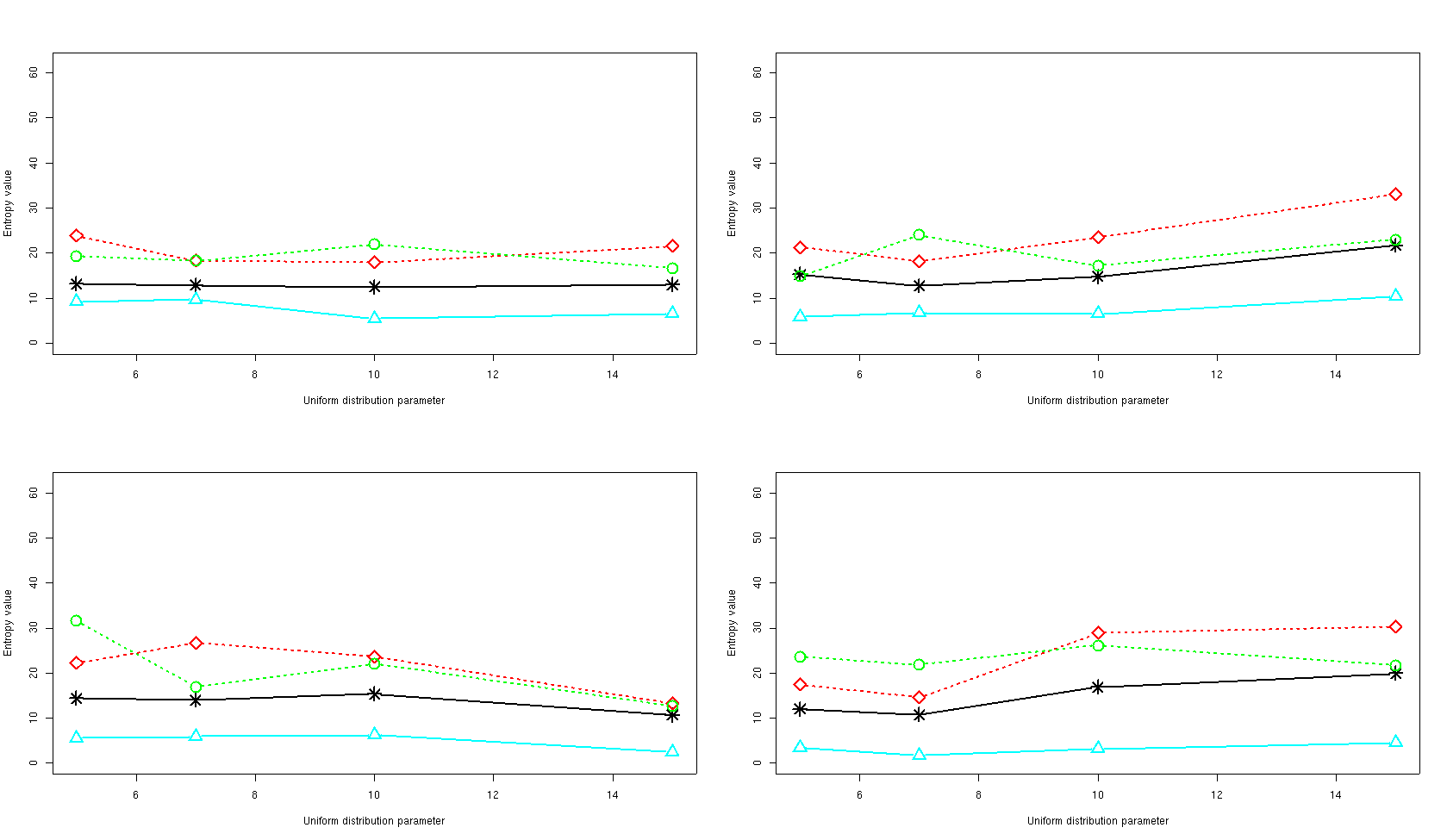}
\caption{Entropy of the weights. Methods: "$\Delta$": PE, "$\nabla$": two-state-HMM, "*": IS, "$\bullet$":  Selected HMM, "$\Diamond$": VB, "O" : Oracle. Top left: $\Pi_{0.05}$, Top right: $\Pi_{0.1}$, Bottom left: $\Pi_{0.2}$, Bottom right: $\Pi_{0.3}$. VB and Oracle are in dotted lines.}
\label{FigEntropy}
\end{center}
\end{figure}

\paragraph{Conclusion on the accuracy of the estimates}

Studying the MSE indicator allows us to compare the methods in terms of classification. Except for the ``two-state-HMM'' approach, we highlight that all the proposed methods have quite similar behaviours. However, the VB method provides better results in terms of MSE and its standard deviation than does the PE approach. These results are very close to those of IS and even often better. The focus on the misclassification rate confirmed the closeness between our approach and that of IS. These methods have a quite similar misclassification rates whatever the simulation condition. Furthermore, this rate corresponds to the best rate among the three averaging approaches. 
The computational time is also a key point of these classification methods. Indeed, the VB method has a negligeable computational time compared with IS. This may further dramatically increase with the size of the data.

\section{Real data analysis}
\label{Appli}
\subsection{Description}

\paragraph{The data}
In this section, we focus on the analysis of a real dataset collected from public health surveillance systems. These data have also been studied in the recent paper of Cai et al. \cite{Cai2009}  using an FDR (False Discovery Rate) approach. The database is composed of 1216 time points. The data and log-transformation of them are shown in figure \ref{Data}.
The event described by the data can be classified into 2 groups: usual or unusual. These two groups correspond to a regular low rate and an irregular high rate respectively. Hence, the first group represents our group of interest and the other one the alternative. Moreover, it is clear that an event highly depends on the past and Strat and Carrat \cite{Strat} demonstrated that this kind of data can be described by using a two-state HMM. In this analysis, we thus aim at retrieving the two groups in the population and we want to estimate well the posterior probability of belonging to the group of interest.  

\paragraph{Initialisation of the algorithm}
To avoid any influence of the prior distributions, they have been chosen as described in Section \ref{Simulation design}. As considered in the simulation section, the alternative distribution has been fitted by a Gaussian mixture with common variance.  The number of components $m$ within the alternative distribution varies from 1 to 6 and the fixed distribution $\mathcal{N}(2.37,0.76^2)$ has been chosen according to results of Cai et al \cite{Cai2009}.

\subsection{Results}

For each number of component we infer the model parameters and estimate the weights with the VB method. The results we obtained are summarized in Table \ref{MixModel}.

\begin{table}
\scriptsize
\begin{center}
\begin{tabular}{|c||c|c|c|c|} 
   \hline
   m  &	mean & variance	& proportions & $\alpha^{VB}$  \\
   \hline
1&$4.9$	&1.1	&$1$	&$<10^{-4}$	\\
   \hline
2&$\left(\begin{array}{cc} 4.5, & 5  \\ \end{array} \right)$ &0.9	&	$\left(\begin{array}{cc} 0.67, & 0.33 \\ \end{array}\right) $ &
$<10^{-4}$	\\
   \hline
3&$\left(\begin{array}{ccc} 4, & 4.2, & 6  \\ \end{array} \right)$ &0.3	&$\left(\begin{array}{ccc} 0.32, & 0.32, & 0.34  \\ \end{array} \right)$ &$0.34$	\\
   \hline
4&$\left(\begin{array}{cccc} 3.9,  & 4.1, &  5, & 6.3 \\ \end{array} \right)$ &0.2	&$\left(\begin{array}{cccc} 0.22, & 0.27, & 0.26,& 0.25  \\ \end{array}\right) $ 	&$0.66$	\\
   \hline
5&$\left(\begin{array}{ccccc} 3.8, & 4, & 4.1, 5.2, & 6.4   \\ \end{array} \right)$&0.18	&$\left(\begin{array}{ccccc} 0.17,& 0.19, & 0.22, & 0.22, & 0.20 \\ \end{array}  \right)$	&$<10^{-4}$	\\
   \hline
6 & $\left( \begin{array}{cccccc} 3.8, & 4, & 4.1, & 4.8, & 5.6, & 6.5 \end{array} \right) $ &0.15	&$\left( \begin{array}{cccccc} 0.14, & 0.16, & 0.16, &0.20, & 0.16, & 0.18\\ \end{array} \right) $ 	&$<10^{-4}$	\\
   \hline
\end{tabular}
\caption{Parameter estimation of the Gaussian mixture within the alternative distribution $f$.}
\label{MixModel}
\end{center}
\normalsize
\end{table}

Every model presented in Table \ref{MixModel} has the same estimation of the transition matrix $\left(\begin{array}{cc} 0.96 & 0.04 \\ 0.04 & 0.96 \\ \end{array} \right)$.
In their article, Cai et al. selected a model with two heterogeneous Gaussian distributions for the alternative. In our approach, due to the homogeneous assumption, the number of components increases and we keep two models with three and four components respectively. The other models have a low weight, smaller than $10^{-4}$, and have no influence on the posterior probability estimation.
We now focus on the classification provided by the averaged distribution and the 3-component model proposed by Cai et al. \cite{Cai2009} and we notice that only 3 points differ between our approach from that of Cai. However if we focus on these three points, we observe that they correspond to points with a posterior probability close to 0.5. These points are on the borderline between the two classes. As our approach tends to increase the posterior probabilities (see figure \ref{AggreVSCai}), the epidemical ranges are greater with our approach. In two cases, the epidemics are declared earlier with the VB method than with that of Cai.

\begin{figure}
\begin{center}
\includegraphics[width=15cm,height=6cm]{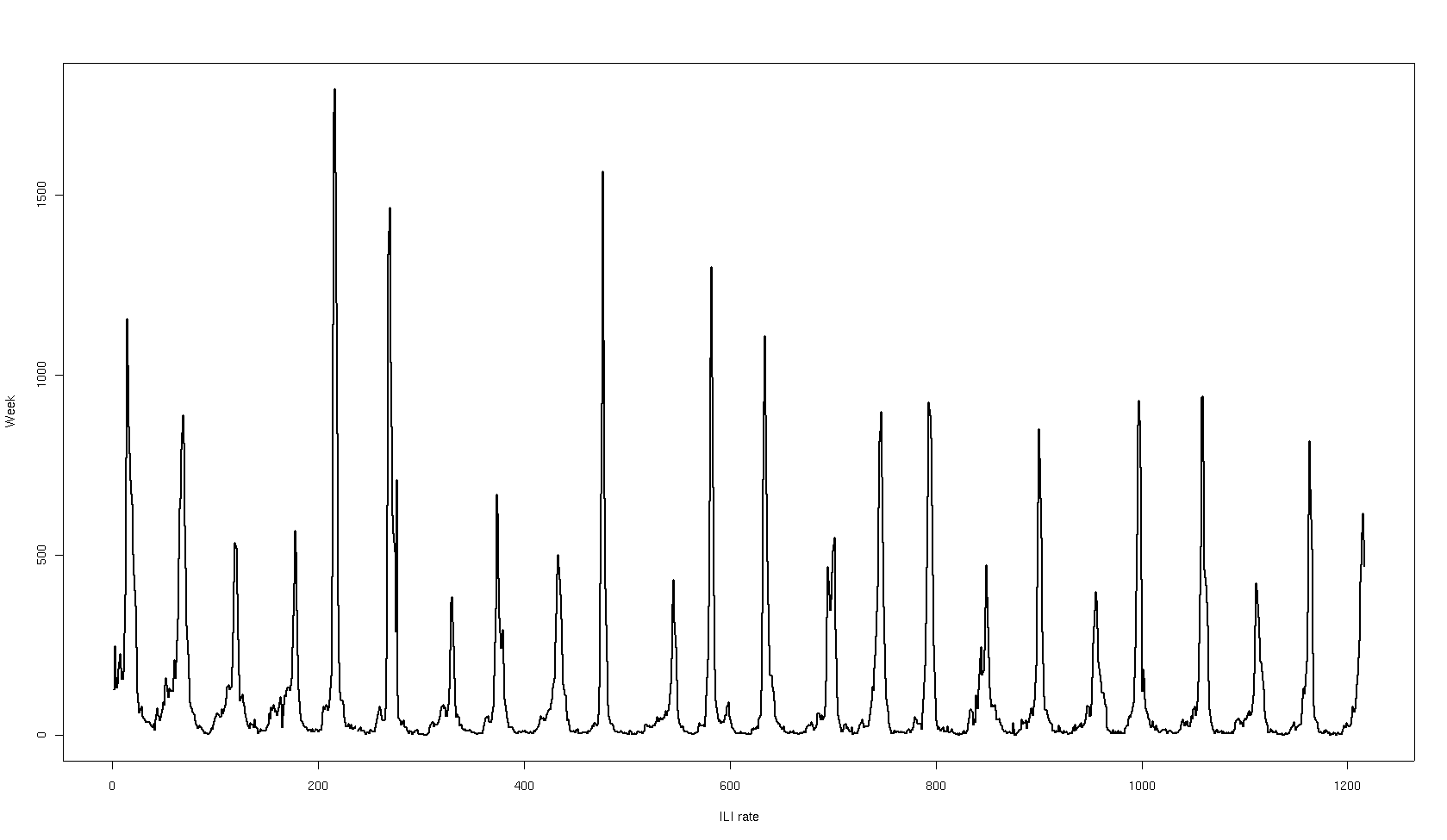}
\includegraphics[height=12cm,width=15cm]{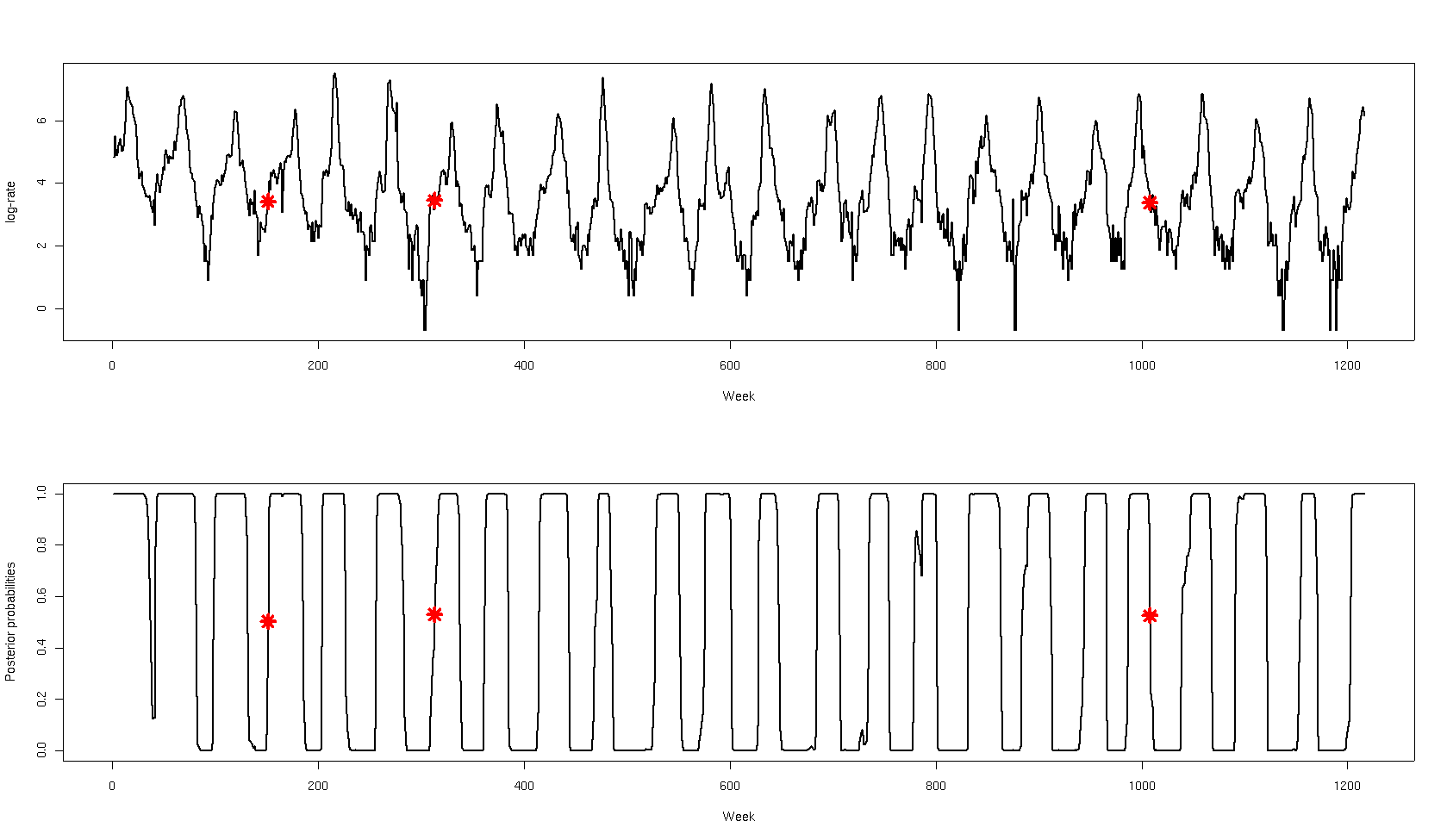}
\caption{top: weekly ILI rate,  middle: log-transformed weekly ILI rate, bottom: Aggregated posterior
probability of ILI epidemic over weeks. The three red points correspond to the points which have a different classification from one method to another.}
\label{Data}
\end{center}
\end{figure}

Figure \ref{AggreVSCai} displays the averaged posterior probabilities against the estimations obtained by the model proposed by Cai. The first comment is that the two approaches provide close estimations. This is especially the case for probabilities smaller than 0.3 or greater than 0.7. These ranges correspond to low entropy areas.
The main comment is that an averaging approach tends to refine posterior probabilities between 0.3 and 0.7. This high entropy area is considered as a difficult area for estimating the probabilities. In fact, it mainly corresponds to data points which are on the borderline between the two classes.

\begin{figure}
\begin{center}
\includegraphics[height=8cm,width=15cm]{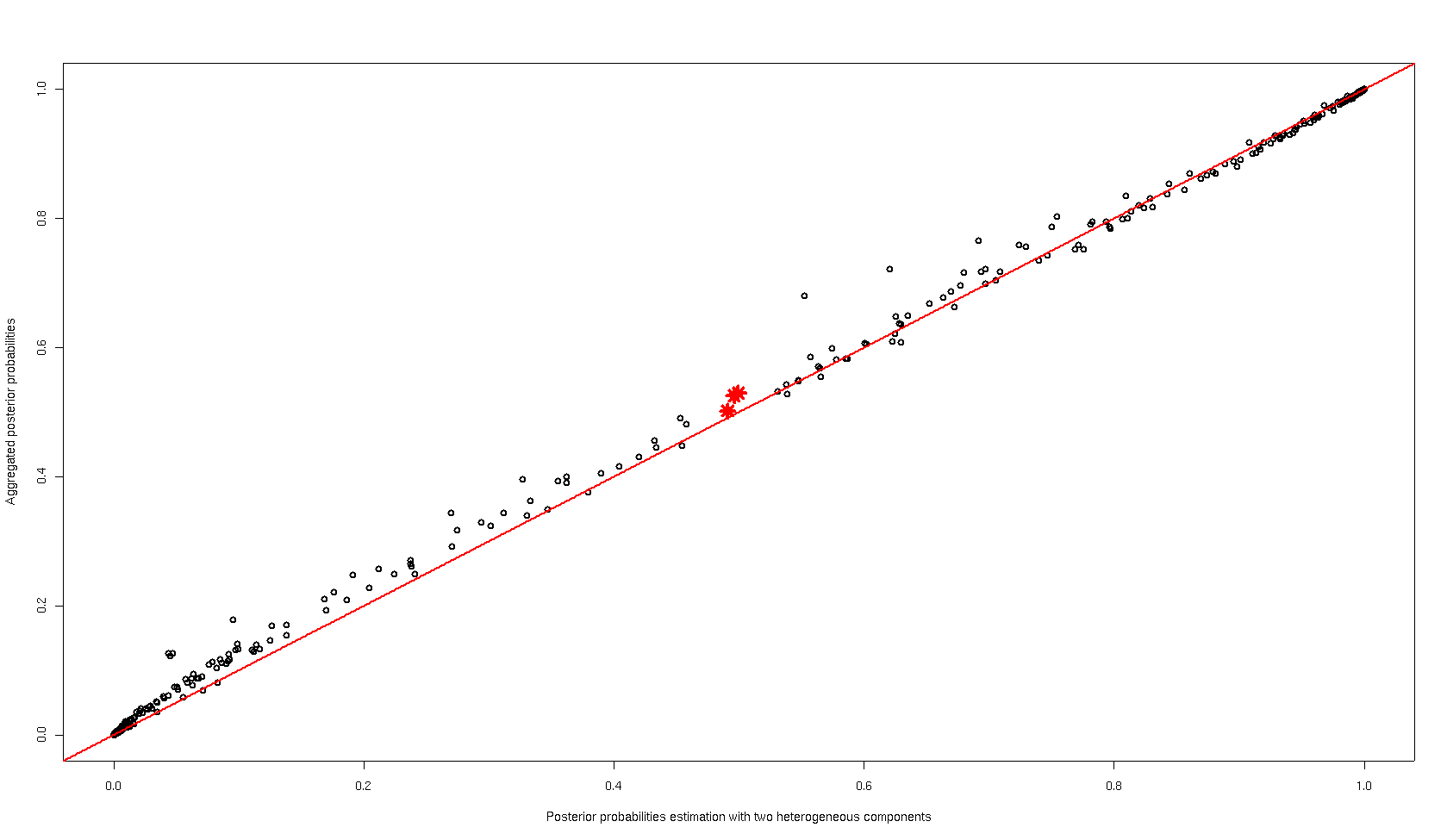}
\caption{Aggregated posterior probabilities according to the estimation of the posterior probabilities with the 2 heterogeneous components model. The three red points correspond to the points which have a different classification from one method to another.}
\label{AggreVSCai}
\end{center}
\end{figure}

\section{Conclusion}

We proposed a method for binary classification problems based on averaged estimators within a variational Bayesian framework. This approach allows us to avoid model selection and take model uncertainty into account. It can theoretically be proved that using an averaged estimator provides a gain in terms of MSE and increases the lower bound of the log-likelihood. We proposed a method based on optimal variational weights which derive from a modification of the classical lower bound of the log-likelihood. Our method does not required more computational time than classical one. For studying the performances, the method has been used on both synthetic and real data.

The results we obtained on synthetic data showed that our method enhances the estimator in terms of MSE in many simulation conditions. We also highlighted that the averaging approach improves the posterior probability estimation provided by the classical selection approach.  Moreover, we showed that optimal variational weights are closer to importance sampling than the plug-in one. Since the importance sampling coped with computational time problems for high dimensional datasets, our method is of significant interest in this case.

A real data analysis has been carried out on a clinical dataset. In this context, the aggregation model still refines the estimation of posterior probabilities. We note in particular that the classification is different in cases where the probability is close to 0.5, i.e. when the classification is difficult. It allows us to refine the start of the epidemic period.

\newpage
\bibliographystyle{plain}
\bibliography{References}

\end{document}